\documentclass[10pt,conference,a4paper]{IEEEtran}
\ifCLASSINFOpdf
\else
\fi
\usepackage{graphicx,subfigure}

\usepackage{amssymb}
\usepackage{amsthm}
\usepackage{amsmath}
\usepackage{algorithm}               
\usepackage{algorithmic}             
\usepackage{url}
\usepackage{color}

\usepackage{graphicx}

\def\cs3c{C$\text{S}^3$C}
\def\s3c{$\text{S}^3$C}
\def\ie{i.e.}
\def\eg{e.g.}

\def\st{\textrm{s.t.}}
\def\diag{\textrm{diag}}

\def\trace{\textrm{trace}}

\def\a{\boldsymbol{a}}

\def\q{\textbf{q}}

\def\0{\textbf{0}}
\def\1{\textbf{1}}
\def\x{\boldsymbol{x}}

\def\C{\mathcal{C}}

\def\E{\mathcal{E}}
\def\EE{\mathbb{E}}

\def\PP{\mathbb{P}}

\def\Q{\mathcal{Q}}

\newcommand{\RR}{I\!\!R} 

\newcommand{\myparagraph}[1]{\smallskip\noindent\textbf{#1.}}

\newtheorem{theorem}{Theorem}[section]

\newtheorem{definition}{Definition}[]

\begin{document}
%

\title{Constrained Structured Sparse Subspace Clustering with Side-Information: A Revisit}
\title{Constrained Structured Sparse Subspace Clustering with Background Knowledge: A Revisit}
\title{Constrained Subspace Clustering with Background Knowledge}
\title{Constrained Structured Subspace Clustering with Background Knowledge}
\title{Structured Subspace Clustering with Background Knowledge}
\title{Structured Subspace Clustering with Side-Information}
\title{Structured Subspace Clustering with Constraints}
\title{Constrained Structured Subspace Clustering with Background Knowledge}
\title{Constrained Structured Sparse Subspace Clustering with Background Knowledge}
\title{Constrained Sparse Subspace Clustering with Background Knowledge}
\title{Constrained Subspace Clustering with Background Knowledge}
\title{Constrained Structured Sparse Subspace Clustering with Side-Information}
\title{Constrained Sparse Subspace Clustering with Side-Information}



%


%

\author{\IEEEauthorblockN{Chun-Guang Li,
Junjian Zhang, and Jun Guo}
\IEEEauthorblockA{School of Information and Communication Engineering\\
Beijing University of Posts and Telecommunications,
Beijing 100876, P.R. China \\
Email: \{lichunguang; zhjj; guojun\}@bupt.edu.cn}
}



\maketitle

\begin{abstract}
Subspace clustering refers to the problem of segmenting high dimensional data drawn from a union of subspaces into the respective subspaces. In some applications, partial side-information to indicate ``must-link'' or ``cannot-link'' in clustering is available. This leads to the task of subspace clustering with side-information. However, in prior work the supervision value of the side-information for subspace clustering has not been fully exploited. To this end, in this paper, we present an enhanced approach for constrained subspace clustering with side-information, termed Constrained Sparse Subspace Clustering plus (CSSC+), in which the side-information is used not only in the stage of learning an affinity matrix but also in the stage of spectral clustering. Moreover, we propose to estimate clustering accuracy based on the partial side-information and theoretically justify the connection to the ground-truth clustering accuracy in terms of the Rand index. We conduct experiments on three cancer gene expression datasets to validate the effectiveness of our proposals.
\end{abstract}
%



%
\IEEEpeerreviewmaketitle

\section{Introduction}
\label{sec:intro}
High dimensional data in many applications can be considered as samples drawn from a union of multiple low-dimensional subspaces. Assigning data points into their own subspaces and then recovering the underlying low-dimensional structure of the data refer to a well-known problem -- \textit{subspace clustering}~\cite{Vidal:SPM11-SC}. It has found important applications in motion segmentation \cite{Rao:PAMI10}, genes expression profiles clustering \cite{McWilliams:DMKD14}, hybrid system identification \cite{Bako:Automatica11}, matrix completion \cite{Li:TSP16}, etc.

\subsection{Prior Work}
\label{sec:prior-work}

Over the past decade, a large number of algorithms have been developed, \eg, $K$-plane \cite{Bradley:JGO00}, Generalized Principal Component Analysis (GPCA) \cite{Vidal:PAMI05}, Spectral Curvature Clustering (SCC) \cite{Chen:IJCV09}, Sparse Subspace Clustering (SSC) \cite{Elhamifar:CVPR09, Elhamifar:TPAMI13, You:CVPR16-SSCOMP}, Low Rank Representation (LRR) \cite{Liu:ICML10, Liu:TPAMI13}, Least Square Regression (LSR) \cite{Lu:ECCV12,Peng:TCYB16}, Correlation Adaptive Subspace Segmentation (CASS) \cite{Lu:ICCV13-TraceLasso}, Latent SSC \cite{Patel:ICCV13}, 
Low-Rank Sparse Subspace Clustering (LRSSC) \cite{Wang:NIPS13-LRR+SSC}, Structured SSC \cite{Li:CVPR15, Li:TIP17, Zhang:VCIP16}, 
and Elastic-net Subspace Clustering (EnSC) \cite{You:CVPR16-EnSC}. 

Among the existing work, self-expressiveness \cite{Elhamifar:CVPR09} based algorithms, \eg, SSC, LRR, LSR, EnSC, gain the most attention. Roughly speaking, different algorithms differ in using different regularization in the self-expressiveness model. For example,  SSC~\cite{Elhamifar:CVPR09} makes use of the $\ell_1$ norm on the coefficients vector, LRR \cite{Liu:ICML10} adopts the nuclear norm on the coefficients matrix, LSR \cite{Lu:ECCV12} uses Frobenius norm on the coefficients matrix, LRSSC \cite{Wang:NIPS13-LRR+SSC} uses a mixture of the $\ell_1$ norm and the nuclear norm on the coefficients matrix, and EnSC \cite{You:CVPR16-EnSC} takes a mixture of the $\ell_1$ and the $\ell_2$ norm on the coefficients vector. On the other hand, 
different error models have also been used to yield robustness, \eg, the $\ell_2$ norm is used to account for the Gaussian noise in data, the $\ell_1$ norm used in SSC to account for the outlying entries in data, the $\ell_{2,1}$ norm used in LRR to account the column-wise corruptions in data, and the mixture of Gaussian model~\cite{Li:CVPR15MoG} and the correntropy \cite{He:TNNLS16} are used to model complicated corruptions. 


The subspace clustering methods mentioned above are purely unsupervised. In some applications, partial supervision information is available. For example, in the task of clustering genes in DNA microarray data \cite{Lockhart:Nature00, Schena:Science95}, there often exists prior knowledge about the relationships between some subset of genes or genes expression profiles \cite{Fang:JBI06, Chopra:BMCbioinfo08, Huang:Bioinfo06, Bair:WIRCS13}. Gene expression data of different cancer subtypes are usually lying on multiple clusters \cite{Alon:PNAS99} and each cluster can be well approximated by a low-dimensional subspace \cite{McWilliams:DMKD14, Cui:PLOS13-LRR}. If some pairs of genes expression profiles are known to have the same (or different) subtypes, then it would be helpful to use this knowledge in subspace clustering. Such prior knowledge essentially provides partial side-information to indicate ``must-link'' or ``cannot-link'' constraints in clustering, which leads to the task of subspace clustering with side-information.

The side-information is important because it provides partial supervision for clustering. Previous work has demonstrated that incorporating side-information in the self-expressiveness model could bring performance improvements. For sample, in \cite{Wang:ICIP14}, some ``must-links'' are encoded into a binary weights matrix to encourage nonzero self-expressive coefficients in the corresponding positions;
in \cite{Li:TIP17}, both ``must-link'' and ``cannot-link'' in side-information are encoded into a weights matrix to encourage or penalize the self-expressiveness coefficients in the corresponding positions.
%
%
While encoding side-information into a weights matrix used in self-expressiveness model could improve the induced affinity, the supervision value of the side-information has not been fully exploited because, encoding pairwise constraints as weights to improve the affinity does not imply that the final clustering must satisfy the constraints. Besides, there is still a reminding issue in prior work on how to select a proper model parameter (\eg, $\lambda$). 

\subsection{Our Contributions}
\label{sec:contribution}

In this paper, we present an enhanced approach for constrained subspace clustering with side-information, termed Constrained Sparse Subspace Clustering plus (CSSC+), 
in which the side-information is used not only in the stage of learning an affinity matrix but also in the stage of spectral clustering. %
In the stage of learning an affinity matrix, each data point is expressed as a linear combination of all other data points, in which the connections to those data points having ``cannot-link'' are inhibited 
and the connections to those data points having ``must-link'' are encouraged. In the stage of spectral clustering, the ``must-link'' constraints and the ``cannot-link'' constraints in side-information are both taken into account into the procedure of clustering. Moreover, we propose to estimate clustering accuracy based on the available side-information and theoretically justify the connection to the ground-truth clustering accuracy in terms of the Rand index. Experiments conducted on three cancer gene expression datasets demonstrate the effectiveness of our proposals.

\myparagraph{Paper Outline} The remainder of the paper is organized as follows. Section~\ref{sec:subspace-clustering-to-constrained-subspace-clustering} gives a review on self-expressiveness based subspace clustering. Section \ref{sec:constrained-sparse-subspace-clustering-complete} presents our proposal---CSSC+. 
Section~\ref{sec:performance-evaluation-estimation} describes clustering quality estimation based on partial side-information.  Section~\ref{sec:experiments} shows experiments and Section~\ref{sec:conclusion} concludes the paper.

\section{From Subspace Clustering to Constrained Subspace Clustering with Side-Information}
\label{sec:subspace-clustering-to-constrained-subspace-clustering}

This section will briefly review methods for subspace clustering and constrained subspace clustering with side-information.

\subsection{Subspace Clustering}
\label{sec:subspace-clustering}

State-of-the-art subspace clustering methods, \eg, SSC \cite{Elhamifar:TPAMI13}, LRR \cite{Liu:TPAMI13}, LSR \cite{Lu:ECCV12}, EnSC \cite{You:CVPR16-EnSC}, are usually based on self-expressiveness model. 
These approaches can be summarized into a unified optimization problem as follows:
\begin{align}
\begin{split}
\min\limits_{C,E} ~ \|C\|_{\C} + \lambda \|E\|_{\E} ~~
\st ~ X = X C + E, ~\text{diag}(C) = \0,
\end{split}
\label{eq:SC-unified-framework}
\end{align}
where $\|\cdot\|_{\C}$ and $\|\cdot\|_{\E}$ are two properly chosen norms, $\lambda > 0$ is a tradeoff parameter, and $\text{diag}(C) = \0$ is optionally used to rule out the trivial solution. Different 
approaches employ different regularization terms $\|C\|_{\C}$ and/or $\|E\|_{\E}$.

Once the optimal representation matrix $C$ is obtained, spectral clustering \cite{vonLuxburg:StatComp2007} can be applied on the induced affinity matrix $A$ where $A = \frac{1}{2}(|C| + |C^\top|)$. Let $Q = \begin{bmatrix}\textbf{q}_1,\cdots,\textbf{q}_n \end{bmatrix}$ be an $N\times n$ indicator matrix where ${q}_{ij}=1$ if the $i$-th column of $X$ lies in subspace $S_j$ and $q_{ij} = 0$ otherwise. Spectral clustering can be formulated as follows: 
\begin{align}
\begin{split}
\min\limits_{Q} ~& \trace \{ Q^\top L Q\} ~~\st~~ Q \in \Q,
\end{split}
\label{eq:Spectral-clustering}
\end{align}
where $L = D - A$, $D$ is a diagonal matrix with $D_{jj}=\sum_i A_{ij}$, and $\Q$ is the set of all valid segmentation matrices with $n$ groups.

\subsection{Structured Subspace Clustering}
\label{sec:structured-subspace-clustering}

Note that the objective function \eqref{eq:Spectral-clustering} of spectral clustering measures the cost of cutting the affinity graph into $n$ parts, and also measures the discrepancy between the coefficient matrix and the segmentation matrix, because
\begin{align}
\begin{split}
\trace \{ Q^\top L Q\} = \sum_{i,j} |C_{ij}| \frac{1}{2} \|\q^{(i)} -\q^{(j)} \|_2^2 = \|C\|_{Q},
\end{split}
\label{eq:C-Q-norm}
\end{align}
where $\| C \|_Q $ is called \textit{subspace structured norm} of representation matrix $C$ with respect to segmentation matrix $Q$ \cite{Li:CVPR15}.

By noticing of the connection between the representation matrix $C$ and the segmentation matrix $Q$, it is natural to integrate problem \eqref{eq:SC-unified-framework} and \eqref{eq:Spectral-clustering} into a joint optimization problem
\begin{align}
\begin{split}
\min\limits_{C,E,Q} ~&
\|C\|_{\C}+\alpha\|C\|_{Q} + \lambda \|E\|_{\E} \\
\st ~~ &X = X C + E, ~\text{diag}(C) = \0, ~~ Q \in \Q,
\end{split}
\label{eq:StrSC}
\end{align}
where $\alpha >0$ is a tradeoff parameter. This is called \textit{structured subspace clustering} \cite{Li:TIP17}. Then, if the $\ell_1$ norm is used for $\|C\|_\C$, problem \eqref{eq:StrSC} turns out to be:
\begin{align}
\begin{split}
\min\limits_{C,E,Q} ~&
\|C\|_{1}+\alpha\|C\|_{Q} + \lambda \|E\|_{\E} \\
\st ~~ &X = X C + E, ~\text{diag}(C) = \0, ~~ Q \in \Q,
\end{split}
\label{eq:S3C}
\end{align}
which is called Structured Sparse Subspace Clustering (\s3c)~\cite{Li:TIP17}.

The algorithms mentioned above are unsupervised. When some side-information is available, there is a desire to incorporate the partial supervision information to facilitate subspace clustering. 

\subsection{Constrained Subspace Clustering with Side-Information}
\label{sec:subspace-clustering}
Recently, an approach, called Constrained Structured Sparse Subspace Clustering (\cs3c) \cite{Li:TIP17} is proposed, in which the side-information is encoded into weights matrix $\Psi$ to modify the $\ell_1$ norm of $C$ (\ie, $\| C \odot \Psi \|_1$), where the operator $\odot$ is the element-wise product and the elements of weights matrix $\Psi$ are defined by:
\begin{align}
\begin{split}
\Psi_{ij} =
\begin{cases}
\text{e}^{-1}, ~~~ \text{if $i$ and $j$ have a ``must-link'',}
\\
\text{e}^{+1}, ~~~ \text{if $i$ and $j$ have a ``cannot-link'',}
\\
\text{e}^{0},~~~~~ \text{if there is no side-information.}
\end{cases}
\end{split}
\label{eq:weight-matrix-Psi}
\end{align}
%
Then, the available side-information is incorporated into the optimization problem \eqref{eq:S3C} as follows:
\begin{align}
\begin{split}
\min\limits_{C,E,Q} ~&
\|C \odot \Psi \|_1 + \alpha \|C\|_Q + \lambda \|E\|_{\E} \\
\st ~~ &X = X C + E, ~\text{diag}(C) = \0, ~~ Q \in \Q.
\end{split}
\label{eq:CS3C}
\end{align}

Encoding the side-information into a weighting matrix $\Psi$ is able to penalize or encourage the coefficients when having ``cannot-link'' or ``must-link'', which thus is helpful to yield an improved coefficients matrix. However, it is not guaranteed that the constraints in side-information could be automatically satisfied in clustering. Besides, it is in principle unclear how to determine the model parameters, \eg, $\lambda$ and $\alpha$ in \eqref{eq:CS3C}. If a set of improper tradeoff parameters are used, the clustering results might dramatically degenerate (see, \eg, Fig.\ref{fig:ERR-vs-EST-Novertis-p5} (b)).
To tackle these deficiencies, in this paper, we propose an enhanced approach for subspace clustering with side-information, in which the side-information is used not only to weight the self-expressiveness model, but also to conduct spectral clustering and parameters selection.

\section{Constrained Sparse Subspace Clustering with Side-Information} 
\label{sec:constrained-sparse-subspace-clustering-complete}

This section will present 
an enhance approach for subspace clustering with side-information.

\subsection{Constrained Sparse Self-Expressiveness Model}
\label{sec:subspace-structured-SR}

In SSC, each data point is expressed as a sparse linear combination of all other data points. To take into account the side-information, we impose the constraints into the self-expressiveness model, such that the connections to those data points having ``cannot-link'' are inhibited and the connections to those data points having ``must-link'' are encouraged.
To be specific, as in \cs3c \cite{Li:TIP17}, we solve for $C$ and $E$ by solving a weighted sparse representation problem as follows
\begin{align}
\label{eq:CSSC}
\!\!
\begin{split}
\min\limits_{C,E} ~& \|C \odot \Psi \|_1 + \lambda \|E\|_{\E}  \\
\st ~~ &X = X C + E, ~ \diag(C) = \0, \!\!
\end{split}
\end{align}
where $\Psi$ is a weights matrix which encodes the available side-information as in \eqref{eq:weight-matrix-Psi}.

We term problem \eqref{eq:CSSC} as Constrained Sparse Subspace Clustering (CSSC). This problem can be solved using the Alternating Direction Method of Multipliers (ADMM) \cite{Lin:09, Lin:NIPS11, Boyd:FTML10}. For the details of the derivation of algorithm to solve problem \eqref{eq:CSSC}, we refer the readers to \cite{Li:TIP17}. 

\subsection{Spectral Clustering with Constraints}
\label{sec:spectral-clustering-constraints}

Given coefficients matrix $C$, we define the data affinity matrix $A$ via $A = \frac{1}{2}(|C|+|C^\top|)$. When partial side-information is available, we impose the constraints into spectral clustering and thus solve spectral clustering with constraints as follows:
\begin{eqnarray}
\begin{array}{rl}
\min\limits_{Q} ~\trace(Q^\top L Q) ~~\text{ s.t.} ~~ Q \in \tilde \Q,
\end{array}
\label{eq:spectral-clustering-with-constraints}
\end{eqnarray}
where $L$ is the graph Laplacian of the data affinity matrix $A$ and $\tilde \Q \subseteq \Q$ is the set of all feasible segmentation matrices $Q$ which satisfy the pairwise constraints encoded in $\Psi$.\footnote{We assume that the constraints in the given side-information are correct and consistent. Thus, there exists a segmentation matrix $Q$ which satisfies all the constraints.}

To solve problem \eqref{eq:spectral-clustering-with-constraints}, we relax the constraint $Q \in \tilde \Q$ to the constraint $Q^\top D Q =I$ and perform spectral embedding at first, \ie, solving
\begin{align}
\min_{Q \in \RR^{N\times n}}  ~\trace(Q^\top L Q) \quad \st \quad Q^\top D Q = I,
\label{eq:SC-Lap-QtDQ}
\end{align}
to find $Q \in \RR^{N \times n}$.  
Then, we quantize $Q$ into the set of feasible segmentation matrices $\tilde \Q$ by applying a constrained $k$-means algorithm \cite{Wagstaff:ICML01}.

We call the two-step approach---first solving the coefficient matrix $C$ via problem \eqref{eq:CSSC} and then finding the segmentation matrix $Q$ via problem \eqref{eq:spectral-clustering-with-constraints} as Constrained Sparse Subspace Clustering plus (CSSC+).

\myparagraph{Remark 1} In \cs3c \cite{Li:TIP17}, instead of searching for a segmentation matrix $Q \in \Q$, we can also search for $Q \in \tilde \Q$, which thus turns problem~\eqref{eq:CS3C} into the following:
\begin{align}
\begin{split}
\min\limits_{C,E,Q} ~&
\| C \odot \Psi \|_1 + \alpha \|C\|_Q + \lambda \|E\|_{\E} \\
\st ~~ &X = X C + E, ~~\text{diag}(C) = \0, ~~ Q \in \tilde{\Q}.
\end{split}
\label{eq:CS3Cplus}
\end{align}
%
We call problem \eqref{eq:CS3Cplus} as \textit{Constrained Structured Sparse Subspace Clustering plus} (\cs3c+). It can be solved by solving subproblems \eqref{eq:CSSC} and \eqref{eq:spectral-clustering-with-constraints} alternatingly.

\section{Performance Evaluation and Estimation with Side-Information}
\label{sec:performance-evaluation-estimation}

This section will present a general method to perform parameter (or model) selection for subspace clustering with side-information.

\subsection{Clustering Error with Respect to Groundtruth Label}
\label{sec:clustering-accuracy}

When the groundtruth label of each data point is available, the quality of clustering can be evaluated by clustering error (ERR),  which is defined as
%
\begin{equation}
ERR(\a, \hat \a) = 1 - \max \limits_\pi \frac{1}{N}\sum_{i=1}^N 1_{\{\pi(\a_i) = \hat \a_i\}},
\end{equation}
where $\a, \hat \a \in \{1,\cdots, n\}^N$ are the original and estimated assignments of the columns in $X$ to $n$ subspaces, and the maximum is taken with respect to all permutations
\begin{align}
\pi:\{1,\cdots, n\}^N\rightarrow \{1,\cdots, n\}^N.
\end{align}

While ERR is a valid measure to compare the partitions of the data with respect to the groundtruth labels, it is not a valid measure to evaluate the partitions of the data with respect to the pairwise side-information (or partial pairwise side-information).

To evaluate the accuracy of clustering with respect to pairwise side-information, we introduce a measure to compare two data partitions, 
which is called Rand index \cite{Rand:JASA71}.

\subsection{Rand Index based on Complete Pairwise Side-Information} 
\label{sec:Rand-index}

Denote $\Theta$ as a subspace structure matrix of the obtained clustering where $\Theta_{ij}=0$ if data points $i$ and $j$ belong to the same cluster and $\Theta_{ij}=1$ otherwise.

\begin{definition}
\label{def:RI}
The Rand index, denoted as $\mu$, is defined as
\begin{align}
\mu =1-\frac{1}{N^2-N} \| \Theta - \Theta^\ast \|_1,
\label{eq:Rand-Index}
\end{align}
where $\Theta$ and $\Theta^\ast$ are the subspace structure matrices of the currently returned clustering and the ground-truth clustering, respectively.
\end{definition}
When the pairwise ground-truth information $\Theta^\ast$ is available, the Rand index is easy to compute. We have that $0 \le \mu \le 1$. If the currently returned clustering is perfect, then $\Theta = \Theta^\ast$ and thus $\mu =1$; otherwise $\mu < 1$.

Since that subspace clustering is an unsupervised task, complete ground-truth knowledge (\eg, $\Theta_\ast$) of the data is unknown. 
Without complete ground-truth knowledge, there is no means to provide a criterion which directly links to the clustering accuracy with theoretical guarantee. Nevertheless, in the setting of clustering with side-information, the side-information is able to provide partial observations of the pairwise ground-truth knowledge. By using the partial observations, there is a hope to define a clustering accuracy estimator, which directly links to the Rand index with theoretical justification.

\subsection{Rand Index Estimator based on Partial Side-Information}
\label{sec:Rand-index-estimator}

\begin{definition}
The Rand Index Estimator (RIE), denoted by $\hat \mu$, is defined as
\begin{align}
\hat \mu = 1 - \frac{1}{ | \Omega |} \sum_{(i,j) \in \Omega} |\Theta_{i,j} - \Phi^\ast_{i,j}|,
\label{eq:Rand-Index-EST}
\end{align}
where $\Theta$ is the subspace structure matrix of the currently returned clustering, $\Omega$ is the index set of the given pairwise constraints in $\Phi^\ast$ in which $\Phi^\ast_{i,j}=0$ if the paired data points $(i,j)$ have a ``must-link'' and $\Phi^\ast_{i,j}=1$ if they have a ``cannot-link''.
\label{def:Rand-Index-EST}
\end{definition}

It is clear that $0 \le \hat \mu \le 1$, where $\hat \mu =1$ if the clustering result indicated by $\Theta$ is feasible with respect to the constraints in $\Phi^\ast$, and $\hat \mu < 1$ otherwise.

\begin{theorem}
Assume that the given constraints in set $\Omega$ are sampled independently at random with probability $p$ from a population of $N(N-1)$ constraints. 
Then, we have that
\begin{align}
| \hat \mu - \mu |  < \frac{2}{pN(N-1)- 1},
\label{eq:Rand-EST-bound}
\end{align}
holds with probability at least $1- 4 e^{-2N(N-1)}$,
where $\mu$ is the Rand index defined in \eqref{eq:Rand-Index} and $\hat \mu$ is the Rand index estimator defined in \eqref{eq:Rand-Index-EST}. 
\label{theorem:mu-hat-mu-bound}
\end{theorem}
\begin{proof}

Let $\{Z_{i,j}\}_{1\le i \le N, 1 \le j \le N, {i \neq j}}$ be $M$ independent identically distributed Bernoulli random variables, where $\PP(Z_{i,j} =1) =p$, $\Omega =\{(i,j): Z_{i,j} =1\}$, and $M = N(N-1)$. 
Denote $Y_{i,j} = Z_{i,j} \Delta_{i,j}$ where $\Delta_{i,j}=|\Theta_{i,j} - \Theta^\ast_{i,j}| \in \{0, 1\}$, then $Y_{i,j} \in \{0, 1\}$ is also a random variable. 

Let $\rho := \frac{1}{M}\sum_{i,j} \Delta_{i,j}$ and $\hat \rho := \frac{1}{\sum_{i,j} Z_{i,j}} \sum_{i,j} Z_{i,j} \Delta_{i,j}$, then we have $\rho = 1- \mu$, $\hat \rho = 1 - \hat \mu$. Moreover, we have
\begin{align}
\EE [\sum_{i,j} Y_{i,j}] =& \EE [\sum_{i,j} Z_{i,j} \Delta_{i,j}], \\
                         =& \sum_{i,j} \EE [Z_{i,j}] \Delta_{i,j}, \\
                         =& p \sum_{i,j} \Delta_{i,j},  \\
                         =& p M \rho,
\label{eq:Rand-EST-Expect-1}
\end{align}
and
\begin{align}
\EE [\sum_{i,j} Z_{i,j}] = \sum_{i,j} \EE [Z_{i,j}] =p M,
\label{eq:Rand-EST-Expect-2}
\end{align}
where $\EE[\cdot]$ is the expectation of a random variable.

By applying Hoeffding's inequality \cite{Hoeffding:JASA63} to $M$ independent random variables $\{Y_{i,j}\}_{i\neq j}$ and $\{Z_{i,j}\}_{i\neq j}$, seperately, we have
\begin{align}
\label{eq:Rand-EST-Hoeffding-Y}
\PP( | \sum_{i,j} Y_{i,j} - p M \rho | > \epsilon ) <  2e^{-2 \epsilon^2 M},
\end{align}
and
\begin{align}
\PP( | \sum_{i,j} Z_{i,j} - p M | > \epsilon ) < 2e^{-2 \epsilon^2 M}.
\label{eq:Rand-EST-Hoeffding-Z}
\end{align}
%
Then, with probability at least $1 -4 e^{-2 \epsilon^2 M}$, we have
\begin{align}
\begin{split}
\label{eq:Rand-EST-Hoeffding-3}
p M \rho - \epsilon &\le \sum_{i,j} Y_{i,j} \le p M \rho + \epsilon,
\end{split}
\end{align}
and
\begin{align}
\begin{split}
pM -\epsilon &\le \sum_{i,j} Z_{i,j} \le p M + \epsilon.
\end{split}
\label{eq:Rand-EST-m-Hoeffding-2}
\end{align}
By combining \eqref{eq:Rand-EST-Hoeffding-3} and \eqref{eq:Rand-EST-m-Hoeffding-2}, we bound $\hat \rho$ as follows:
\begin{align}
\begin{split}
\frac{p M \rho - \epsilon}{p M + \epsilon} \le \hat \rho =\frac{\sum_{i,j} Y_{i,j}}{\sum_{i,j} Z_{i,j}} \le \frac{p M \rho + \epsilon}{pM -\epsilon},
\end{split}
\label{eq:Rand-EST-Hoeffding}
\end{align}
\ie,
\begin{align}
\begin{split}
\!\!\!\frac{-2\epsilon}{p M - \epsilon} \le \frac{-\epsilon(1+\rho)}{p M + \epsilon} \le \hat \rho - \rho  \le \frac{\epsilon(1+\rho)}{pM -\epsilon}\le \frac{2 \epsilon}{pM -\epsilon}.
\end{split}
\label{eq:Rand-EST-Hoeffding2}
\end{align}
So, we have
\begin{align}
\begin{split}
|\hat \rho - \rho |  \le \frac{2 \epsilon}{pM -\epsilon},
\end{split}
\label{eq:Rand-EST-Hoeffding3}
\end{align}
holds with probability at least $1 -4e^{-2 \epsilon^2 M}$. 

Note that $M=N(N-1), \hat \rho - \rho = \mu - \hat \mu$, and by taking $\epsilon=1$, then we have that:
\begin{align}
\begin{split}
| \hat \mu - \mu | < \frac{2}{pN(N-1) - 1},
\end{split}
\label{eq:Rand-EST-Hoeffding4}
\end{align}
holds with probability at least $1 -4e^{-2 N(N-1)}$. This completes the proof.


\end{proof}

\begin{figure}[tbh]
\centering
\subfigure[RIE]{\includegraphics[clip=true,trim=5 0 22 5,width=0.49\columnwidth]{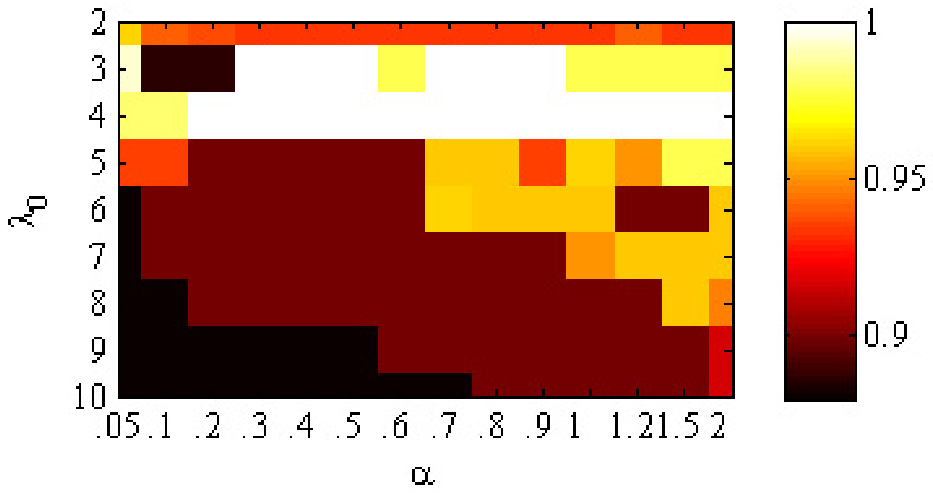}}
\subfigure[ERR]{\includegraphics[clip=true,trim=5 0 22 5,width=0.49\columnwidth]{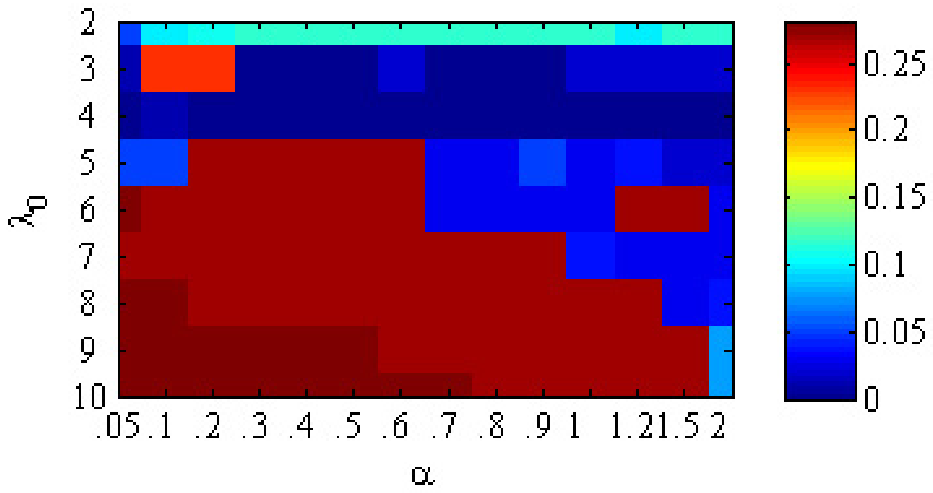}}
\caption{The Rand index estimator (RIE) and the clustering error (ERR) for \cs3c on dataset Novartis BPLC with $5\%$ side-information. (This figure is best viewed in color.)}
\label{fig:ERR-vs-EST-Novertis-p5}
\end{figure}

\myparagraph{Remark 2} If the available side-information is sampled at random and sufficient, the Rand index estimator could provide a good estimation for the Rand index, which connects to the true clustering accuracy. Nevertheless, in case of that the side-information is neither sufficient nor sampled at random, the Rand index estimator might fail to give an acceptable estimation for the Rand index.

\subsection{Parameter Selection via the Rand Index Estimator} 
\label{sec:parameter-selection-via-rand-est}

When the side-information are sufficient and sampled at random, we use the Rand index estimator $\hat \mu$ to estimate the Rand index.
\begin{table}[hbt]
\centering
\caption{Summary Information of Datasets. $D$ is the ambient dimension, $N$ is the number of data points, and $n$ is the number of groups.}
\label{tab:dataset}
\begin{tabular}{c | c  c  c  c | c c }
\hline
Data sets &St. Jude leukemia &Lung Cancer  &Novartis BPLC   \\
\hline
$D$       &    985           &1000         &1000       \\ 
$N$       &    248           &197          &103      \\ 
$n$       &    6             &4            &4        \\ 
\hline
\end{tabular}
\end{table}
\begin{table*}[thb]
\caption{Average clustering error (ERR \%) with standard derivation (std) on three cancer data sets under different proportions of side-information. The best results are in bold font.} 
\centering
\begin{tabular}{c|c c c | c c c | c c c }
\hline
Data sets        & \multicolumn{3}{c|}{St. Jude leukemia}        & \multicolumn{3}{c|}{Lung Cancer}     & \multicolumn{3}{c}{Novartis BPLC}      \\
Side-info.       & 0\%       & 1\%        &5\%       &0 \%         & 1\%        &5\%       & 0\%    & 1\%        &5\%    \\
\hline
\hline
PSC(1) \cite{McWilliams:DMKD14} &3.10  &    -  &  - &7.80   &  -   &-    &4.60  &  -   &  -   \\
\hline
LRR~\cite{Liu:ICML10}              &14.11 &    -     &     -         &5.08     & -             &  -          &14.60  &  -  &  -           \\
\!\!\!\!\!\!\!\!\!LRR+             &14.11 &18.51$\pm$4.02&13.27$\pm$3.87 &5.08     &6.73$\pm$3.02  &1.55$\pm$1.88&14.60  &9.13$\pm$3.01&  4.76$\pm$1.75             \\

\hline

LSR1 \cite{Lu:ECCV12}              &9.27   &      &    -        &4.57& -             &  -          &6.80 &  -      & -       \\
\!\!\!\!\!\!\!\!\!LSR1+              &9.27  &15.52$\pm$1.86 &8.65$\pm$2.97&4.57&10.23$\pm$2.64 &3.50$\pm$3.01&6.80 &7.28$\pm$1.20 &4.56$\pm$2.50     \\

LSR2 \cite{Lu:ECCV12}               &9.68  &  -            &    -        &4.57& -             &  -            &7.77 &  -          & -        \\
\!\!\!\!\!\!\!\!\!LSR2+              &9.68  &15.08$\pm$1.80 &9.82$\pm$1.95&4.57&11.60$\pm$3.88 &4.77$\pm$3.03  &7.77 &7.18$\pm$1.20& 3.59$\pm$1.94     \\

\hline
\hline
\!\!SSC \cite{Elhamifar:CVPR09}           &3.23    &    -         &  -            &5.08 & -           &  -         &\textbf{2.91}&  -          &  -      \\

\!\!\!\!\!\!\!\!SSC+           &3.23    &2.80$\pm$0.92 &1.03$\pm$0.69  &5.08 &4.80$\pm$1.41&1.45$\pm$2.02 &\textbf{2.91}&\textbf{1.55$\pm$1.11}&  \textbf{0.24$\pm$0.53}               \\

\hline 

\!\!\!\!\!\!\!\!CSSC            &3.23    &2.38$\pm$0.60 &1.43$\pm$0.50  &5.08 &4.59$\pm$2.64&4.44$\pm$2.80 &\textbf{2.91}&2.82$\pm$0.43   &2.52$\pm$1.11\\

\!\!\!\!\!CSSC+           &3.23    &\textbf{2.08$\pm$0.95} &\underline{0.56$\pm$0.50}  &5.08 &\textbf{3.02$\pm$0.99}&\underline{1.02$\pm$1.22} &\textbf{2.91}&\underline{1.60$\pm$0.96}   &{0.44$\pm$0.80}\\

\hline

\text{\cs3c~\cite{Li:TIP17}}   &\textbf{2.42}    &2.32$\pm$0.54&1.33$\pm$0.49  &\textbf{4.06}&3.98$\pm$0.30 & 3.81$\pm$0.48 &\textbf{2.91} &2.77$\pm$0.48 &2.33$\pm$1.02 \\ 

\hline

\!\!\!\!\!\!\!\!\!\text{\cs3c+}   & \textbf{2.42}      &\textbf{2.08$\pm$0.98}  &\textbf{0.48$\pm$0.43}&\textbf{4.06}&\underline{3.60$\pm$1.37} &\textbf{0.84$\pm$0.98}  &\textbf{2.91} &\underline{1.60$\pm$0.96} &\underline{0.34$\pm$0.65}     \\

\hline
\end{tabular}
\label{table:CS3Cplus-Leukemia-LungA-Novartis}
\end{table*}
As an interesting application, we employ the Rand index estimator $\hat \mu$ to select the proper parameters\footnote{It is also possible to use the Rand index estimator to determine the number of subspaces when it is unknown.}, \eg, $\alpha$ and $\lambda$ in \eqref{eq:CS3C}, $\lambda$ in \eqref{eq:CSSC}. More specifically, we conduct experiments with the parameters varying in a range, record the Rand index estimator $\hat \mu$ as a function of the parameters, and then pick up the parameters which associate to the peak value of the Rand index estimator.

An example of parameters selection with the Rand index estimator is shown in Fig.~\ref{fig:ERR-vs-EST-Novertis-p5}, where panel (a) shows the Rand index estimator of \cs3c as a function of the parameters $\lambda_0$ and $\alpha$, and panel (b) shows the corresponding clustering error. The coordinates of the brightest region in panel (a) indicate the potentially proper parameters. As verified by experiments, the corresponding dark blue region in panel (b) did yield the lowest clustering error.

\section{Experiments}
\label{sec:experiments}
This section will evaluate the effectiveness of our proposals for subspace clustering with side-information.

We consider three publicly available benchmark cancer datasets\footnote{http://www.broadinstitute.org/cgi-bin/cancer/datasets.cgi}: St. Jude leukemia \cite{Yeoh:CC02-abbr}, Lung Cancer \cite{Bhattacharjee:PNAS01-abbr}, and Novartis BPLC \cite{Su:PNAS02-abbr}. For clarity, we list the summary information of the three datasets in Table \ref{tab:dataset}.
To prepare the side-information, following the protocol used in \cite{Li:TIP17}, we sample uniformly at random a proportion $p$ of entries from the ground-truth subspace structure matrix $\Theta^\ast$, where $p=1\%$ and $5\%$.

\subsection{Performance Evaluation on Subspace Clustering with Side-Information}
\label{sec:experiments-clustering-with-constraints}

To evaluate the performance of using side-information, we choose three popular spectral clustering based methods: SSC \cite{Elhamifar:TPAMI13}, LRR \cite{Liu:ICML10, Cui:PLOS13-LRR}, LSR \cite{Lu:ECCV12}, and a PCA based subspace clustering method, Predictive Subspace Clustering (PSC) \cite{McWilliams:DMKD14}. Moreover, we conduct experiments to compare the following approaches: 
\cs3c \cite{Li:TIP17}, SSC+, LRR+, LSR+, and CS$^3$C+, where the appendix ``+'' means using the $k$-means with constraints in spectral clustering. Note that if the percentage of given side-information is $0\%$, then \cs3c and CSSC reduce to \s3c and SSC, respectively.

\begin{table}[htb]
\caption{Parameter $\lambda$ or $\lambda_0$ used in each method on each dataset.}
\centering
\begin{tabular}{c|c c c c }
\hline
Methods       & LRR ($\lambda$) & LSR1 ($\lambda$) & LSR2 ($\lambda$) & SSC ($\lambda_0$)   \\
\hline
St. Jude leukemia       &1.4       &0.15       &0.18      & 4  \\
Lung Cancer          &0.5       &5       &5      & 10 \\
Novartis BPLC  &1       &1       &1      & 5 \\
\hline
\end{tabular}
\label{table:parameters}
\end{table}

The average clustering error (ERR) with standard deviation is recorded over 20 trials. Experimental results are presented in Table \ref{table:CS3Cplus-Leukemia-LungA-Novartis}. 
The results of PSC are directly cited from \cite{McWilliams:DMKD14}. The parameter $\lambda$ used in each baseline method is listed  in Table \ref{table:parameters}, where the parameter 
$\lambda$ used in the family of SSC, including \cs3c, CSSC, CSSC+, and \cs3c+, 
is kept the same as in SSC by default.\footnote{For \cs3c, once an exceptionally worse clustering result occurs, we use the Rand index estimator to tune the parameters $\alpha$ and $\lambda$. The parameter $\alpha$ in \cs3c and \cs3c+ is kept the same.}
We observe that:
\begin{itemize}

\item When the side-information is relatively sufficient, \eg, $p=5\%$, the clustering errors of all methods with side-information  
    are significantly reduced, compared to the counterpart method without side-information. 
    This hints the importance of incorporating the constraints to clustering.

\item When the side-information is relatively not sufficient, \eg, $p=1\%$, the clustering errors of SSC+, CSSC+, and CS$^3$C+ are still notably reduced compared to SSC, CSSC, and CS$^3$C; however, the clustering errors of LRR+, LSR1+, and LSR2+ are exceptionally increased, respectively, in most cases. This suggests that the effect of imposing constraints in clustering depends on the quality of the affinity matrix.

\item Compared to SSC, both CSSC and CSSC+ reduce the clustering errors notably. Comparing to CSSC, CSSC+ reduces the clustering error more significantly. It is similar for \cs3c and \cs3c+. This confirms the superiority of incorporating the constraints into not only the self-expressiveness model but also the process of clustering.


\end{itemize}

\begin{figure*}[ht]
\centering
\vspace{-0mm}
\subfigure[CSSC ($1\%$)]{\includegraphics[clip=true,trim=0 0 5 0,width=0.45\columnwidth]{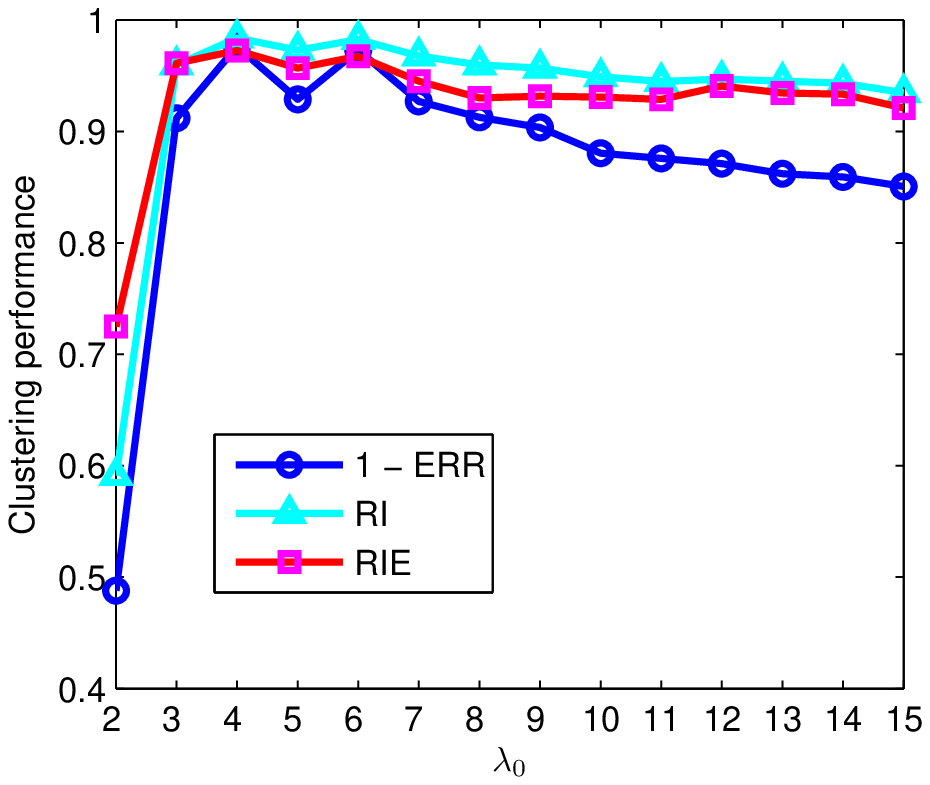}}
\subfigure[CSSC+ ($1\%$)]{\includegraphics[clip=true,trim=0 0 5 0,width=0.45\columnwidth]{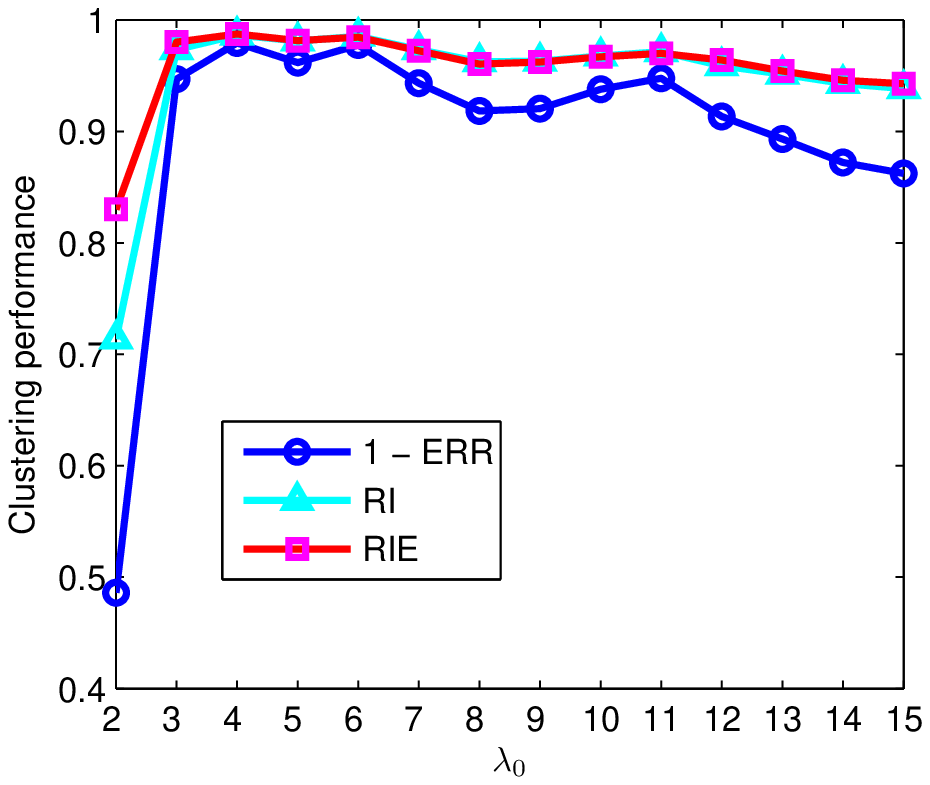}}
\subfigure[CSSC ($5\%$)]{\includegraphics[clip=true,trim=0 0 5 0,width=0.45\columnwidth]{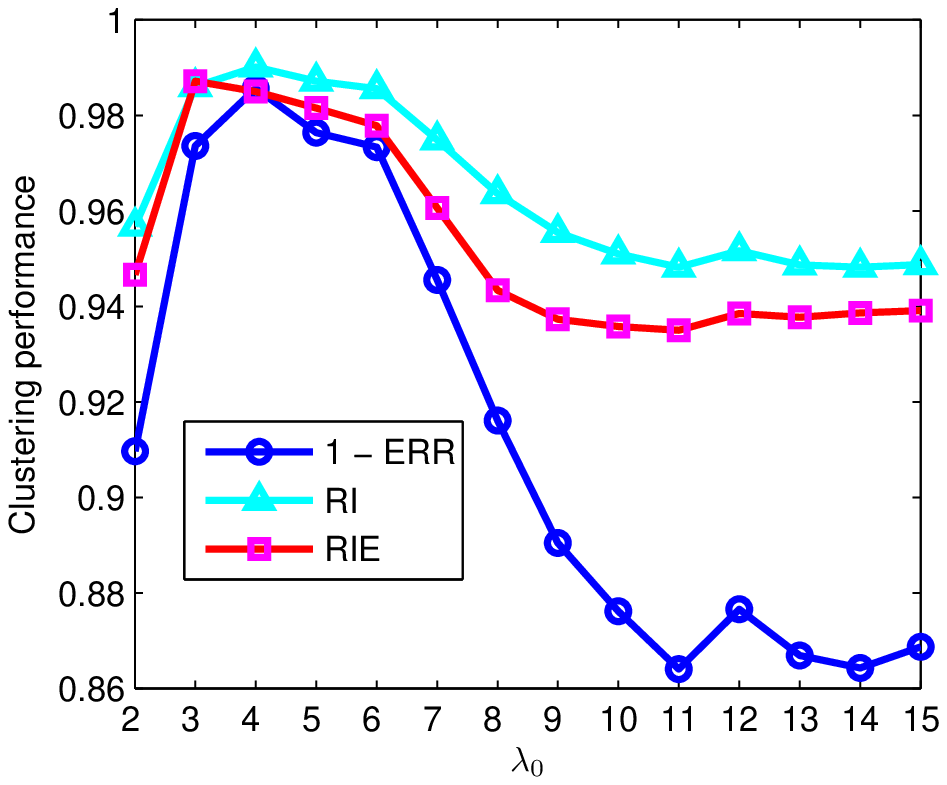}}
\subfigure[CSSC+ ($5\%$)]{\includegraphics[clip=true,trim=0 0 5 0,width=0.45\columnwidth]{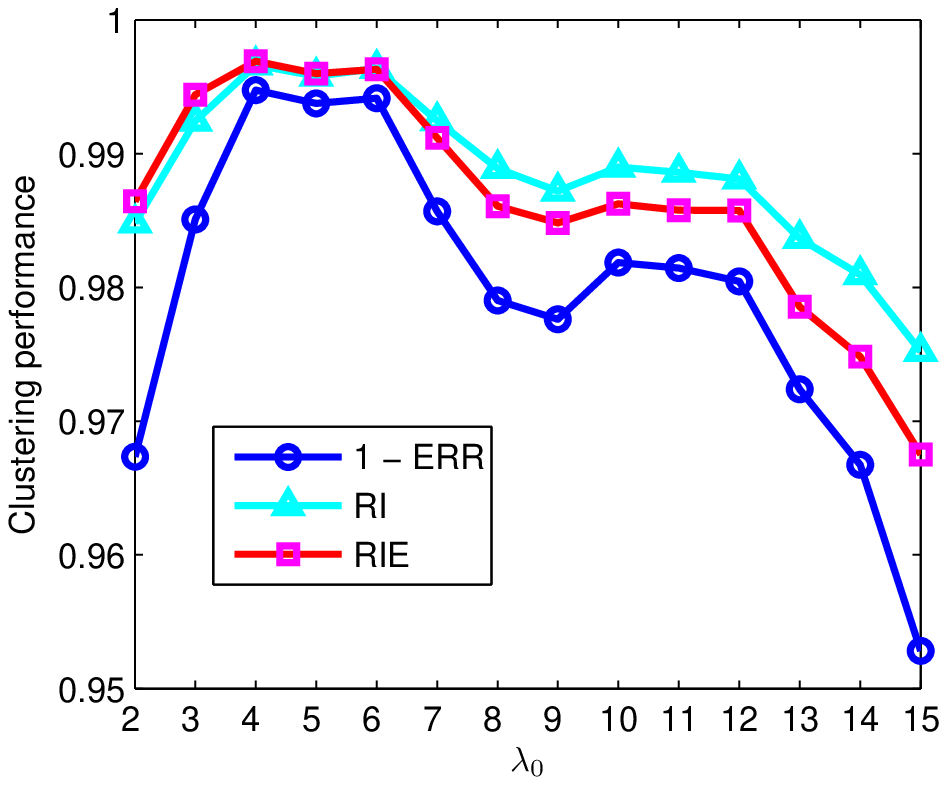}}\\

\subfigure[CSSC ($1\%$)]{\includegraphics[clip=true,trim=0 0 5 0,width=0.45\columnwidth]{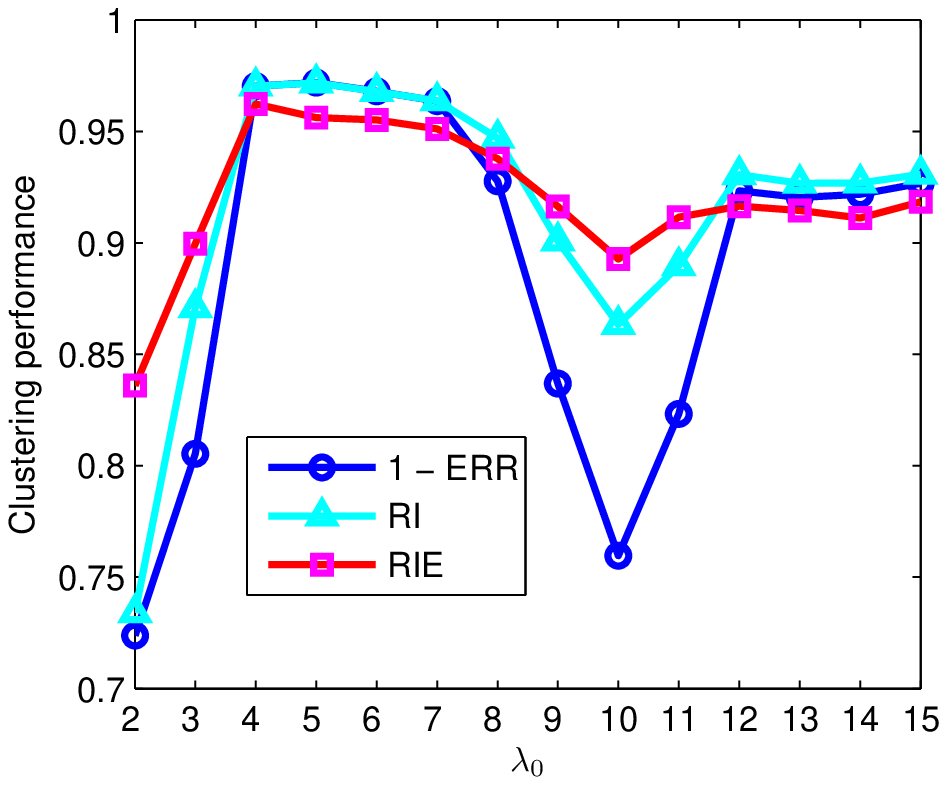}}
\subfigure[CSSC+ ($1\%$)]{\includegraphics[clip=true,trim=0 0 5 0,width=0.45\columnwidth]{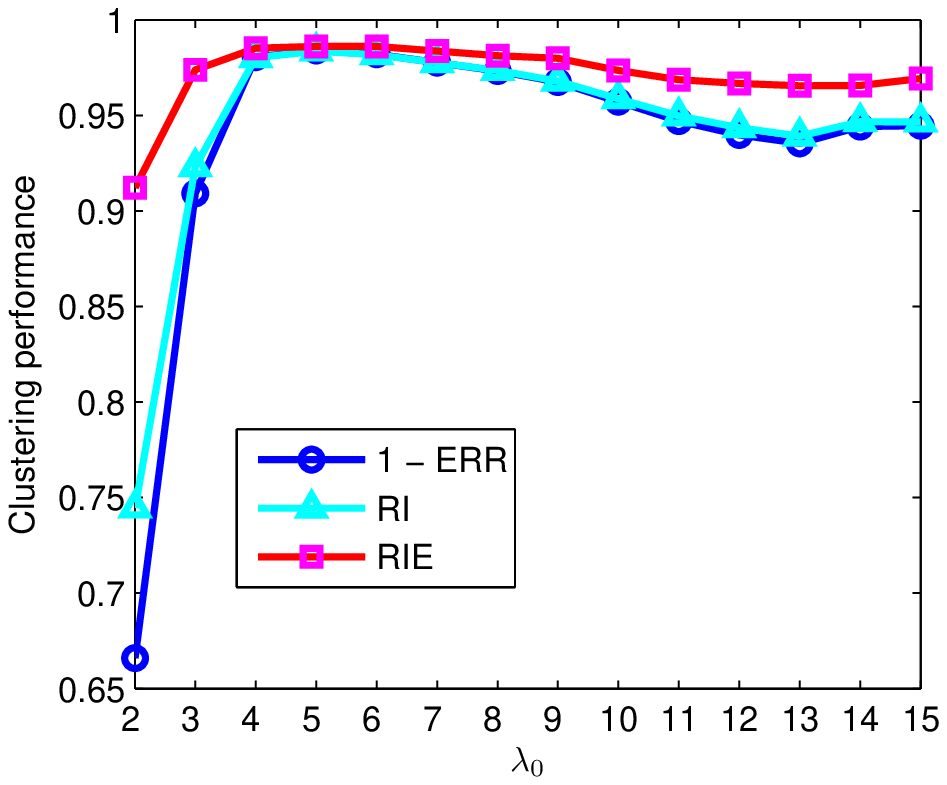}}
\subfigure[CSSC ($5\%$)]{\includegraphics[clip=true,trim=0 0 5 0,width=0.45\columnwidth]{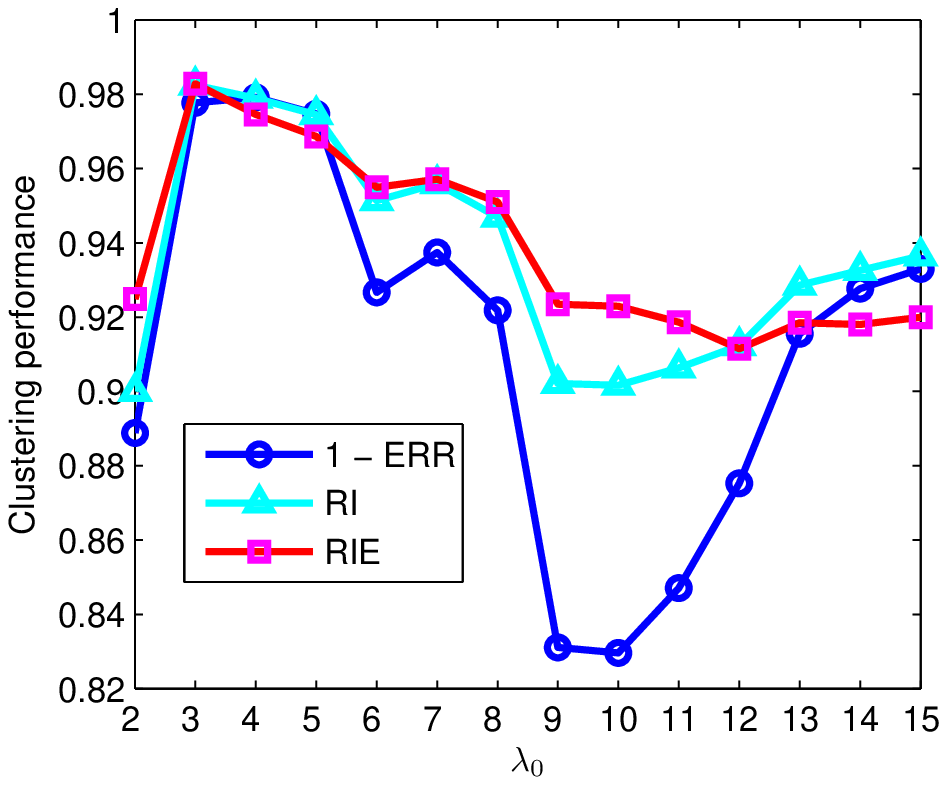}}
\subfigure[CSSC+ ($5\%$)]{\includegraphics[clip=true,trim=0 0 5 0,width=0.45\columnwidth]{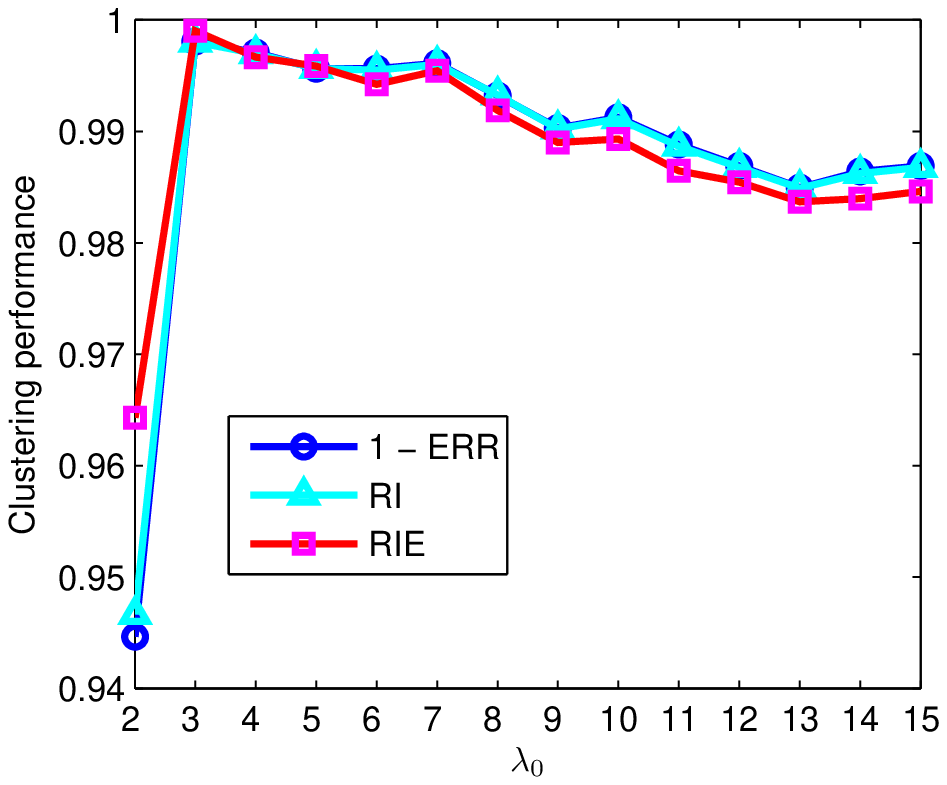}}\\

\subfigure[CSSC ($1\%$)]{\includegraphics[clip=true,trim=0 0 5 0,width=0.45\columnwidth]{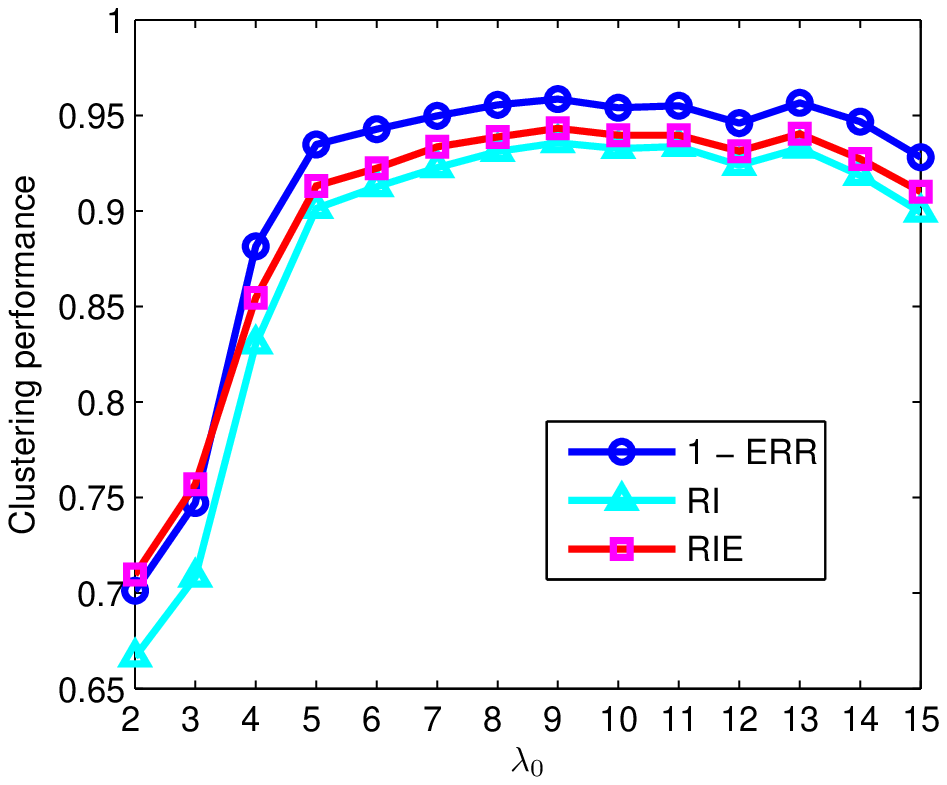}}
\subfigure[CSSC+ ($1\%$)]{\includegraphics[clip=true,trim=0 0 5 0,width=0.45\columnwidth]{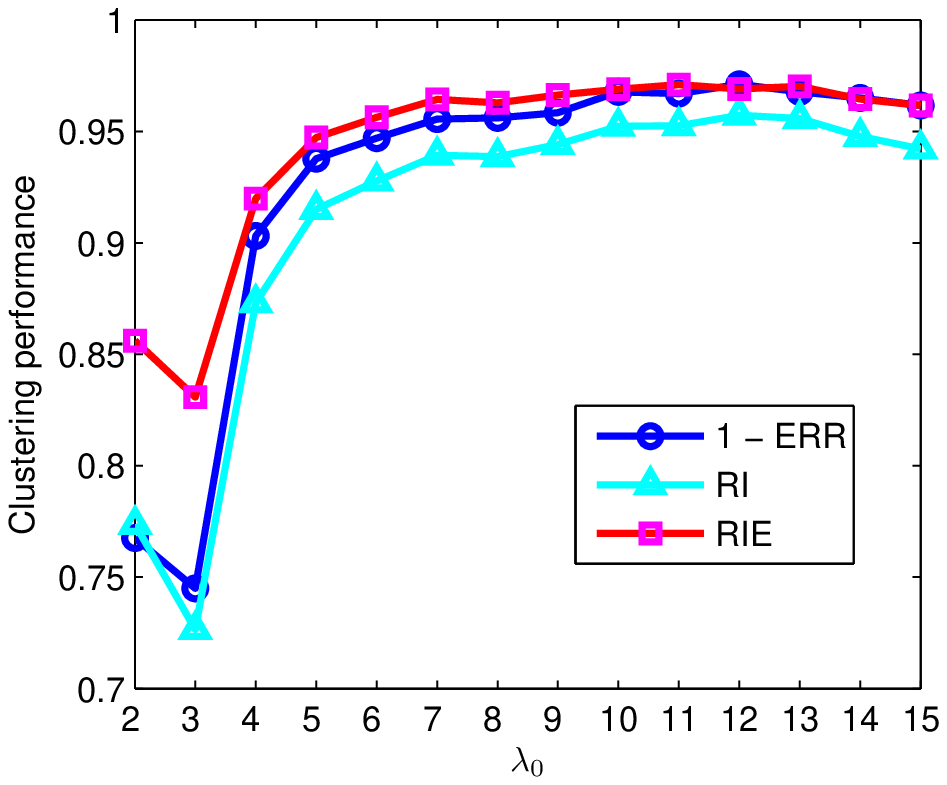}}
\subfigure[CSSC ($5\%$)]{\includegraphics[clip=true,trim=0 0 5 0,width=0.45\columnwidth]{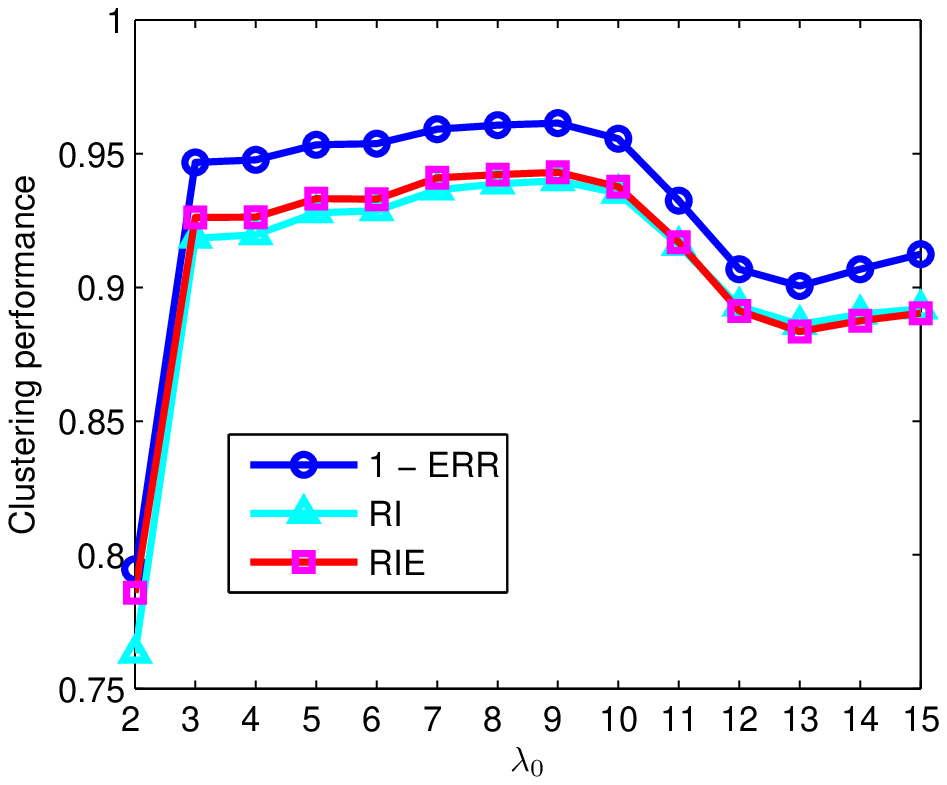}}
\subfigure[CSSC+ ($5\%$)]{\includegraphics[clip=true,trim=0 0 5 0,width=0.45\columnwidth]{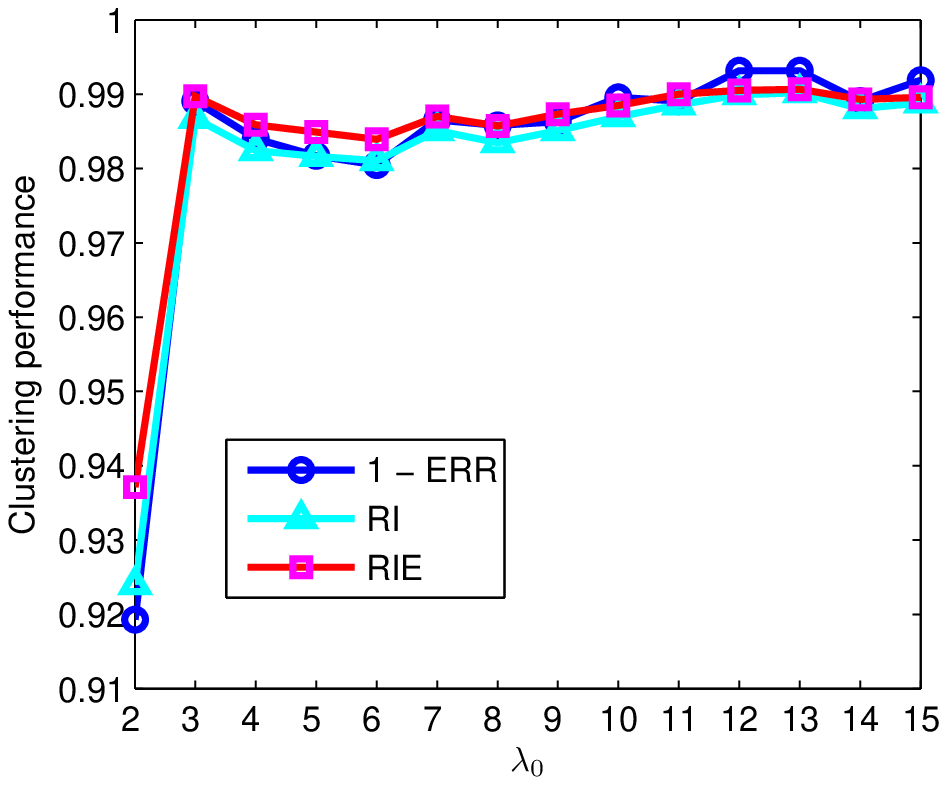}}

\caption{The clustering accuracy (1 - ERR), the Rand index (RI), and the Rand index estimator (RIE) on each dataset: St. Jude leukemia (top row), Lung Cancer (middle row) and Novartis BPLC (bottom row), where the percentage in bracket is the proportion of the available side-information.}
\label{fig:ERR-vs-EST-Novertis-p5-p1-CS3C-CS2Cplus}
\end{figure*}

\subsection{Parameter Selection via Rand Index Estimator}
\label{sec:experiments-parameter-selection}

To demonstrate the feasibility of using the Rand index estimator to guide parameters selection, we conduct experiments on dataset Novartis BPLC with \cs3c by varying parameters $\lambda$ and $\alpha$, where
$\alpha \in \{0.05, 0.1, 0.2, \cdots, 0.9, 1.0, 1.2, 1.5, 2.0\}$ and $\lambda$ is set by
\begin{align}
\lambda =  \frac{\lambda_0}{\min_j \max_{i:i \neq j} \{\x_i^\top \x_j\}}
\end{align}
with $\lambda_0 \in  \{2,3,4,5,6,7,8,9,10\}$. The clustering error (ERR) and the Rand index estimator (RIE) under all combination of parameters are recorded.
We show ERR and RIE in Fig.\ref{fig:ERR-vs-EST-Novertis-p5} panels (a) and (b), respectively, as a function of parameters $\alpha$ and $\lambda_0$. As could be observed that, the change patterns of the two panels are consistently correlated. This confirms the feasibility of using the Rand index estimator to conduct parameters selection in \cs3c.

Moreover, we also conduct experiments for CSSC and CSSC+ on all three datasets. We show the clustering accuracy (1 - ERR), the Rand index (RI), and the Rand index estimator (RIE), respectively, as a function of the parameter $\lambda_0$, in Fig.\ref{fig:ERR-vs-EST-Novertis-p5-p1-CS3C-CS2Cplus}. As could be observed, the positions of the peaks of the Rank index estimator correspond to the best clustering accuracy. Compared to CSSC, the Rand index estimator is more consistent to the true clustering accuracy (in terms of both the Rand index and the clustering error). This confirms that the Rand index estimator is of practical value to conduct parameter selection, especially for CSSC+. 

\section{Conclusion}
\label{sec:conclusion}

We have presented an enhanced 
framework to perform constrained subspace clustering with side-information, in which the constraints in side-information are used not only in the stage of learning the affinity matrix but also in the stage of spectral clustering. Moreover, we have proposed an Rand index estimator based on partial side-information for estimating the clustering accuracy with theoretical guarantee and used it to conduct parameters selection. Experiments on three cancer gene expression datasets have validated the effectiveness of our proposals.

The Rand index estimator is a general measure for estimating the clustering accuracy with partial pairwise side-information, not limited to subspace clustering. More comprehensive evaluations on the performance of constrained subspace clustering with side-information, model selection with the Rand index estimator, and more active way \cite{Li:ICCV15-STSSL} to exploit the partial supervision of side-information will be our future work.

\section*{Acknowledgment}
C.-G. Li is partially supported by the Open Project Fund from Key Laboratory of Machine Perception (MOE) , Peking University. C.-G. Li would like to thank Chong You for his valuable comments, especially on the proof of Theorem~\ref{theorem:mu-hat-mu-bound}.


\small
\bibliographystyle{IEEEtran}
\bibliography{biblio/temp,biblio/cgli,biblio/vidal,biblio/vision,biblio/math,biblio/learning,biblio/sparse,biblio/geometry,biblio/dti,biblio/recognition,biblio/surgery,biblio/coding,biblio/matrixcompletion,biblio/segmentation,biblio/computationalbiology}

\begin{thebibliography}{10}
\providecommand{\url}[1]{#1}
\csname url@samestyle\endcsname
\providecommand{\newblock}{\relax}
\providecommand{\bibinfo}[2]{#2}
\providecommand{\BIBentrySTDinterwordspacing}{\spaceskip=0pt\relax}
\providecommand{\BIBentryALTinterwordstretchfactor}{4}
\providecommand{\BIBentryALTinterwordspacing}{\spaceskip=\fontdimen2\font plus
\BIBentryALTinterwordstretchfactor\fontdimen3\font minus
  \fontdimen4\font\relax}
\providecommand{\BIBforeignlanguage}[2]{{%
\expandafter\ifx\csname l@#1\endcsname\relax
\typeout{** WARNING: IEEEtran.bst: No hyphenation pattern has been}%
\typeout{** loaded for the language `#1'. Using the pattern for}%
\typeout{** the default language instead.}%
\else
\language=\csname l@#1\endcsname
\fi
#2}}
\providecommand{\BIBdecl}{\relax}
\BIBdecl

\bibitem{Vidal:SPM11-SC}
R.~Vidal, ``Subspace clustering,'' \emph{{IEEE} Signal Processing Magazine},
  vol.~28, no.~3, pp. 52--68, March 2011.

\bibitem{Rao:PAMI10}
S.~Rao, R.~Tron, R.~Vidal, and Y.~Ma, ``Motion segmentation in the presence of
  outlying, incomplete, or corrupted trajectories,'' \emph{{IEEE} Transactions
  on Pattern Analysis and Machine Intelligence}, vol.~32, no.~10, pp.
  1832--1845, 2010.

\bibitem{McWilliams:DMKD14}
B.~McWilliams and G.~Montana, ``Subspace clustering of high dimensional data: a
  predictive approach,'' \emph{Data Mining and Knowledge Discovery}, vol.~28,
  no.~3, pp. 736--772, 2014.

\bibitem{Bako:Automatica11}
L.~Bako, ``Identification of switched linear systems via sparse optimization,''
  \emph{Automatica}, vol.~47, no.~4, pp. 668--677, 2011.

\bibitem{Li:TSP16}
C.-G. Li and R.~Vidal, ``A structured sparse plus structured low-rank framework
  for subspace clustering and completion,'' \emph{{IEEE} Transactions on Signal
  Processing}, vol.~64, no.~24, pp. 6557--6570, 2016.

\bibitem{Bradley:JGO00}
P.~S. Bradley and O.~L. Mangasarian, ``k-plane clustering,'' \emph{Journal of
  Global Optimization}, vol.~16, no.~1, pp. 23--32, 2000.

\bibitem{Vidal:PAMI05}
R.~Vidal, Y.~Ma, and S.~Sastry, ``{Generalized Principal Component Analysis
  (GPCA)},'' \emph{{IEEE} Transactions on Pattern Analysis and Machine
  Intelligence}, vol.~27, no.~12, pp. 1--15, 2005.

\bibitem{Chen:IJCV09}
G.~Chen and G.~Lerman, ``Spectral curvature clustering ({SCC}),''
  \emph{International Journal of Computer Vision}, vol.~81, no.~3, pp.
  317--330, 2009.

\bibitem{Elhamifar:CVPR09}
E.~Elhamifar and R.~Vidal, ``Sparse subspace clustering,'' in \emph{{IEEE}
  Conference on Computer Vision and Pattern Recognition}, 2009, pp. 2790--2797.

\bibitem{Elhamifar:TPAMI13}
------, ``Sparse subspace clustering: Algorithm, theory, and applications,''
  \emph{{IEEE} Transactions on Pattern Analysis and Machine Intelligence},
  vol.~35, no.~11, pp. 2765--2781, 2013.

\bibitem{You:CVPR16-SSCOMP}
C.~You, D.~Robinson, and R.~Vidal, ``Scalable sparse subspace clustering by
  orthogonal matching pursuit,'' in \emph{{IEEE} Conference on Computer Vision
  and Pattern Recognition}, 2016, pp. 3918--3927.

\bibitem{Liu:ICML10}
G.~Liu, Z.~Lin, and Y.~Yu, ``Robust subspace segmentation by low-rank
  representation,'' in \emph{International Conference on Machine Learning},
  2010, pp. 663--670.

\bibitem{Liu:TPAMI13}
G.~Liu, Z.~Lin, S.~Yan, J.~Sun, and Y.~Ma, ``Robust recovery of subspace
  structures by low-rank representation,'' \emph{IEEE Transactions on Pattern
  Analysis and Machine Intelligence}, vol.~35, no.~1, pp. 171--184, Jan 2013.

\bibitem{Lu:ECCV12}
C.-Y. Lu, H.~Min, Z.-Q. Zhao, L.~Zhu, D.-S. Huang, and S.~Yan, ``Robust and
  efficient subspace segmentation via least squares regression,'' in
  \emph{European Conference on Computer Vision}, 2012, pp. 347--360.

\bibitem{Peng:TCYB16}
X.~Peng, Z.~Yu, Z.~Yi, and H.~Tang, ``Constructing the $l_2$-graph for robust
  subspace learning and subspace clustering,'' \emph{IEEE Transactions on
  Cybernetics}, vol.~47, no.~4, pp. 1053--1066, 2017.

\bibitem{Lu:ICCV13-TraceLasso}
C.~Lu, Z.~Lin, and S.~Yan, ``Correlation adaptive subspace segmentation by
  trace lasso,'' in \emph{{IEEE} International Conference on Computer Vision},
  2013, pp. 1345--1352.

\bibitem{Patel:ICCV13}
V.~M. Patel, H.~V. Nguyen, and R.~Vidal, ``Latent space sparse subspace
  clustering,'' in \emph{{IEEE} International Conference on Computer Vision},
  2013, pp. 225--232.

\bibitem{Wang:NIPS13-LRR+SSC}
Y.-X. Wang, H.~Xu, and C.~Leng, ``Provable subspace clustering: When {LRR}
  meets {SSC},'' in \emph{Neural Information Processing Systems}, 2013.

\bibitem{Li:CVPR15}
C.-G. Li and R.~Vidal, ``Structured sparse subspace clustering: A unified
  optimization framework,'' in \emph{Proceedings of {IEEE} International
  Conference on Computer Vision and Pattern Recognition}, 2015, pp. 277--286.

\bibitem{Li:TIP17}
C.-G. Li, C.~You, and R.~Vidal, ``Structured sparse subspace clustering: A
  joint affinity learning and subspace clustering framework,'' \emph{{IEEE}
  Transactions on Image Processing}, vol.~26, no.~6, pp. 2988--3001, 2017.

\bibitem{Zhang:VCIP16}
J.~Zhang, C.-G. Li, H.~Zhang, and J.~Guo, ``Low-rank and structured sparse
  subspace clustering,'' in \emph{IEEE International Conference on Visual
  Communication and Image Processing (VCIP)}, Nov. 2016.

\bibitem{You:CVPR16-EnSC}
C.~You, C.-G. Li, D.~Robinson, and R.~Vidal, ``Oracle based active set
  algorithm for scalable elastic net subspace clustering,'' in
  \emph{Proceedings of {IEEE} International Conference on Computer Vision and
  Pattern Recognition}, 2016, pp. 3928--3937.

\bibitem{Li:CVPR15MoG}
B.~Li, Y.~Zhang, Z.~Lin, and H.~Lu, ``Subspace clustering by mixture of
  gaussian regression,'' in \emph{{IEEE} Conference on Computer Vision and
  Pattern Recognition}, 2015, pp. 2094--2102.

\bibitem{He:TNNLS16}
R.~He, L.~Wang, Z.~Sun, Y.~Zhang, and B.~Li, ``Information theoretic subspace
  clustering,'' \emph{IEEE Ttransactions on Neural Networks and Learning
  Systems}, vol.~27, no.~12, pp. 2643--2655, 2016.

\bibitem{Lockhart:Nature00}
D.~Lockhart and E.~Winzeler, ``Genomics, gene expression, and dna arrays,''
  \emph{Nature}, vol. 405, pp. 827--836, 2000.

\bibitem{Schena:Science95}
M.~Schena, D.~Shalon, R.~Davis, and P.~Brown, ``Quantitative monitoring of gene
  expression patterns with a complementary dna microarray,'' \emph{Science},
  vol. 270, pp. 467--470, 1995.

\bibitem{Fang:JBI06}
Z.~Fang, J.~Yang, Y.~Li, Q.~Luo, L.~Liu, and et~al., ``Knowledge guided
  analysis of microarray data,'' \emph{Journal of Biomedical Informatics},
  vol.~39, no.~4, pp. 401--411, 2006.

\bibitem{Chopra:BMCbioinfo08}
P.~Chopra, J.~Kang, J.~Yang, H.~Cho, H.~Kim, and M.~Lee, ``Microarray data
  mining using landmark gene-guided clustering,'' \emph{BMC Bioinformatics},
  vol.~9, no.~1, p.~92, 2008.

\bibitem{Huang:Bioinfo06}
D.~Huang and W.~Pan, ``Incorporating biological knowledge into distance-based
  clustering analysis of microarray gene expression data,''
  \emph{Bioinformatics}, vol.~22, no.~10, pp. 1259--1268, 2006.

\bibitem{Bair:WIRCS13}
E.~Bair, ``Semi-supervised clustering methods,'' \emph{Wiley Interdisciplinary
  Reviews: Computational Statistics}, vol.~5, no.~5, pp. 349--361, 2013.

\bibitem{Alon:PNAS99}
U.~Alon, N.~Barkai, D.~Notterman, K.~Gish, S.~Ybarra, D.~Mack, and A.~Levine,
  ``Broad patterns of gene expression revealed by clustering analysis of tumor
  and normal colon tissues probed by oligonucleotide arrays.''
  \emph{Proceedings of the National Academy of Sciences of the United States of
  America}, vol.~96, no.~12, pp. 6745--6750, 1999.

\bibitem{Cui:PLOS13-LRR}
Y.~Cui, C.-H. Zheng, and J.~Yang, ``Identifying subspace gene clusters from
  microarray data using low-rank representation,'' \emph{Plos ONE}, vol.~8,
  no.~3, p. e59377, 2013.

\bibitem{Wang:ICIP14}
D.~Wang, Q.~Yin, R.~He, L.~Wang, and T.~Tan, ``Semi-supervised subspace
  segmentation,'' in \emph{IEEE International Conference on Image Processing},
  2014, pp. 2854--2858.

\bibitem{vonLuxburg:StatComp2007}
U.~von Luxburg, ``A tutorial on spectral clustering,'' \emph{Statistics and
  Computing}, vol.~17, no.~4, pp. 395--416, 2007.

\bibitem{Lin:09}
Z.~Lin, M.~Chen, L.~Wu, and Y.~Ma, ``The augmented {Lagrange} multiplier method
  for exact recovery of corrupted low-rank matrices,''
  \emph{arXiv:1009.5055v2}, 2011.

\bibitem{Lin:NIPS11}
Z.~Lin, R.~Liu, and Z.~Su, ``Linearized alternating direction method with
  adaptive penalty for low rank representation,'' in \emph{Neural Information
  Processing Systems}, 2011.

\bibitem{Boyd:FTML10}
S.~Boyd, N.~Parikh, E.~Chu, B.~Peleato, and J.~Eckstein, ``Distributed
  optimization and statistical learning via the alternating direction method of
  multipliers,'' \emph{Foundations and Trends in Machine Learning}, vol.~3,
  no.~1, pp. 1--122, 2010.

\bibitem{Wagstaff:ICML01}
K.~Wagstaff, C.~Cardie, S.~Rogers, and S.~Schr\"{o}dl, ``Constrained $k$-means
  clustering with background knowledge,'' in \emph{ICML}, 2001, pp. 577--584.

\bibitem{Rand:JASA71}
W.~Rand, ``Objective criteria for the evaluation of clustering methods,''
  \emph{Journal of the American Statistical Association}, vol.~66, pp.
  846--850, 1971.

\bibitem{Hoeffding:JASA63}
W.~Hoeffding, ``Probability inequalities for sums of bounded random
  variables,'' \emph{Journal of the American Statistical Association}, vol.~58,
  no. 301, pp. 13--30, 1963.

\bibitem{Yeoh:CC02-abbr}
E.~J. Yeoh, M.~E. Ross, S.~A. Shurtleff, and et~al., ``Classification, subtype
  discovery, and prediction of outcome in pediatric acute lymphoblastic
  leukemia by gene expression profiling,'' \emph{Cancer Cell}, vol.~1, pp.
  133--143, 2002.

\bibitem{Bhattacharjee:PNAS01-abbr}
A.~Bhattacharjee, W.~Richards, J.~Staunton, and et~al., ``Classification of
  human lung carcinomas by mrna expression profiling reveals distinct
  adenocarcinomas sub-classes,'' \emph{Proceedings of the National Academy of
  Sciences of the United States of America}, vol.~98, no.~24, pp.
  13,790--13,795, 2001.

\bibitem{Su:PNAS02-abbr}
A.~I. Su, M.~P. Cooke, K.~A. Ching, and et~al., ``Large-scale analysis of the
  human and mouse transcriptomes,'' \emph{Proceedings of the National Academy
  of Sciences of the United States of America}, vol.~99, no.~7, pp. 4447--4465,
  2002.

\bibitem{Li:ICCV15-STSSL}
C.-G. Li, Z.~Lin, H.~Zhang, and J.~Guo, ``Learning semi-supervised
  representation towards a unified optimization framework for semi-supervised
  learning,'' in \emph{Proceedings of {IEEE} International Conference on
  Computer Vision}, 2015, pp. 2767--2775.

\end{thebibliography}

\end{document}